\documentclass[11pt]{article}

\usepackage[margin=1in]{geometry}
\usepackage[utf8]{inputenc}
\usepackage[T1]{fontenc}
\usepackage{xcolor}
\usepackage{amsmath,amssymb}
\usepackage{graphicx}
\usepackage{booktabs}
\usepackage{multirow}
\usepackage{float}
\usepackage{algorithm}
\usepackage{algorithmic}
\usepackage{enumitem}
\usepackage{hyperref}
\usepackage{url}

\setlist{itemsep=0pt}

\newcommand{\R}{\mathbb{R}}
\newcommand{\E}{\mathbb{E}}
\newcommand{\doX}{\mathrm{do}}

\title{Mathematical Reasoning via Intervention-Based Causal Discovery
Using LLMs as Concept Mastery Simulators\\[4pt]
\normalsize Technical Report}

\author{
  Tsuyoshi Okita\\
  Department of Artificial Intelligence, Kyushu Institute of Technology\\
  \texttt{tsuyoshi@ai.kyutech.ac.jp}
}

\date{}

\begin{document}
\maketitle

\begin{abstract}
A large language model (LLM) often holds the mathematical knowledge that a problem requires without using that knowledge in its answer. Test-time search and causal-graph-guided prompting improve accuracy, but neither tests whether a given concept causally raises the probability of a correct answer, and the observed association between a concept and correctness can be driven by problem difficulty. This report proposes CIKA, which uses the frozen LLM itself as an interventional simulator. A prompt that sets the state of one concept to ``mastered'' is treated as an intervention, and the resulting change in the correctness rate, called the Interventional Capability Probe (ICP), estimates the effect of that concept. CIKA also splits each problem into smaller subproblems, mostly two, and combines ICP estimation, a causal-UCB tree search and recovery from several perspectives. With Qwen2.5-Math-7B-Instruct, the top-ranked concept on 67 screened MATH problems has an ICP of $+0.219$ against $+0.039$ for an unrelated control concept ($p < 10^{-6}$). On Omni-MATH-Rule, accuracy rises from 35.8\% for a direct answer by majority vote to 69.7\%. This design costs more than 110 LLM calls per problem, and this report presents the results of this design direction together with that cost.
\end{abstract}

\section{Introduction}
\label{sec:intro}

A large language model trained on mathematical text frequently fails on a problem whose solution uses a concept that the same model can state and apply in isolation. The model knows the concept but does not use the concept when the concept is needed. Two lines of work have improved LLM mathematical reasoning at test time. The first line spends more inference on search. MCTS-based methods such as rStar-Math~\cite{guan2025rstar}, CMCTS~\cite{huang2025cmcts} and PaCoRe~\cite{jiang2026pacore} raise the accuracy of 7B-parameter models substantially, but the search does not reveal which piece of knowledge made an answer correct. The second line injects structured knowledge. CAMA~\cite{cama2025} builds a Mathematical Causal Graph from data in advance and supplies the graph to the model as a prompt. The graph is fixed before the problem is seen, and no intervention is performed on the model.

Both lines rely on association between a concept and a correct answer. Association does not establish that the concept contributes to correctness, because a common cause can produce the association. Problem difficulty $D$ is such a cause, since a hard problem makes a concept less likely to be used and also makes a correct answer less likely. Removing this kind of confounding requires an intervention in the sense of Pearl~\cite{pearl2009causality}.\footnote{The intervention $\doX(X=x)$ sets $X$ to $x$ and removes every causal influence on $X$. Conditioning on $X=x$ keeps those influences. The back-door criterion~\cite{pearl2009causality} states when a causal effect can be computed from data, and an intervention satisfies the criterion by removing all arrows into $X$.} In physical domains, prior structural VAR-based causal discovery obtains interventions from a simulator that can be run with a variable held fixed.\footnote{A structural vector autoregressive (SVAR) model decomposes a multivariate time series into contemporaneous causal effects ($B_0$) and lagged dynamics ($B_\ell$)~\cite{sims1980macroeconomics,kilian2017structural}.} An LLM can play a similar role for its own knowledge, because the input of the model can be written so that the state of a concept is fixed from outside.

Interventions of this kind have a price. The effect of a concept is estimated from repeated samples, and the estimate is needed for every candidate concept. A problem with many candidate concepts is costly to probe, and a natural response is to split the problem into smaller subproblems, each of which needs fewer concepts. The split and the probing add stages to every problem that the model fails, and the number of LLM calls per problem grows accordingly. The present report documents a method built in this direction, together with its cost. The contributions are as follows.

\begin{itemize}
\item The report treats the frozen LLM as an interventional simulator, defines the Interventional Capability Probe (ICP) as the change in correctness rate under a concept intervention, and shows on 67 screened problems that the top-ranked concept has an ICP of $+0.219$ against $+0.039$ for an unrelated control concept ($p < 10^{-6}$, Cohen's $d = 0.86$).
\item The report describes CIKA, which splits a problem into smaller subproblems, activates concepts with positive ICP and guides a tree search with the ICP values. With a 7B-parameter model whose weights are frozen, CIKA reaches 69.7\% on Omni-MATH-Rule and 64.0\% on the full Omni-MATH set, and Causal Knowledge Activation adds correct answers on 33.8\% of Omni-MATH-Rule problems that the direct answer fails.
\item The report gives the call count of the method as a sum over its stages, more than 110 LLM calls per problem in the reported setting, and relates the count to the success probability of independent sampling at the same budget.
\end{itemize}

\section{Related Work}
\label{sec:related}

Table~\ref{tab:positioning} compares CIKA with representative methods.

\paragraph{MCTS-based LLM mathematical reasoning:}
MCTS-based reasoning is a central approach to test-time compute scaling. Representative methods include rStar-Math~\cite{guan2025rstar}, which combines MCTS with a Process Reward Model,\footnote{A Process Reward Model scores each intermediate reasoning step, whereas an Outcome Reward Model scores only the final answer.} CMCTS~\cite{huang2025cmcts}, PaCoRe~\cite{jiang2026pacore}, LLaMA-Berry~\cite{zhang2024llamaberry} and R$^2$-LLMs~\cite{chen2025r2llms}. These methods use UCT or UCB1~\cite{auer2002finite,silver2016mastering} as the selection rule.

\paragraph{Causal bandits and causal reinforcement learning:}
Lattimore et al.~\cite{lattimore2016causal} introduced causal bandits, in which knowledge of the causal structure reduces the number of interventions needed to find a good arm. Subsequent work covers counterfactual fairness~\cite{huang2022ducb}, unknown graphs with back-door adjustment~\cite{chen2025baucb} and causal Markov decision processes~\cite{lu2021causal}, and a survey is given by Subramanian and Ravindran~\cite{huang2025causalbanditsurvey}.

\paragraph{Causal inference and mathematical reasoning:}
CAMA~\cite{cama2025} constructs a Mathematical Causal Graph by statistical causal discovery before test time and injects the graph as a prompt. The graph is shared across problems, no intervention is performed and problem difficulty is not treated as a confounder. Surveys~\cite{liu2025causalllmsurvey} and benchmark studies~\cite{kasetty2024interventional} report that LLMs reason about interventions only to a limited degree. The present report does not ask the model to reason about interventions. The model is the system on which interventions are performed.

\paragraph{Test-time adaptation:}
Test-time training~\cite{sun2020ttt,kang2024testtime} updates model weights at test time. Self-Consistency~\cite{wang2023selfconsistency} and Best-of-N with a verifier~\cite{lightman2023lets} draw many reasoning paths. AdaReasoner~\cite{wang2025adareasoner} learns reasoning configurations such as temperature and the number of steps with a reinforcement learning policy. CKA needs neither gradient computation nor majority voting.

\begin{table}[t]
\centering
\small
\caption{Comparison with related methods. LLM Sim. means that the LLM itself is used as a do-operator simulator. Prob.-Specific means that a causal structure is built for each problem at test time. Confound means that a latent confounder is treated explicitly. FT-Free means that no fine-tuning is required.}
\label{tab:positioning}
\begin{tabular}{lcccccc}
\toprule
Method & MCTS & Causal Graph & LLM Sim. & Prob.-Specific & Confound & FT-Free \\
\midrule
rStar-Math~\cite{guan2025rstar} & $\bigcirc$ & $\times$ & $\times$ & n/a & $\times$ & $\times$ \\
CMCTS~\cite{huang2025cmcts} & $\bigcirc$ & $\times$ & $\times$ & n/a & $\times$ & $\bigcirc$ \\
CAMA~\cite{cama2025} & $\times$ & $\bigcirc$ & $\times$ & $\times$ (static) & $\times$ & $\bigcirc$ \\
C-UCB~\cite{lattimore2016causal} & $\times$ & $\bigcirc$ & $\times$ & n/a & $\bigcirc$ & n/a \\
PaCoRe~\cite{jiang2026pacore} & $\times$ & $\times$ & $\times$ & n/a & $\times$ & $\times$ \\
CIKA (ours) & $\bigcirc$ & $\bigcirc$ & $\bigcirc$ & $\bigcirc$ (dynamic) & $\bigcirc$ & $\bigcirc$ \\
\bottomrule
\end{tabular}
\end{table}

\section{Problem Setting and Notation}
\label{sec:problem}

Let a problem $P$ with correct answer $y^*$ be given. An LLM $\mathcal{M}$ with frozen parameters produces an answer $\hat{y} = \mathcal{M}(P, \mathcal{C})$ for a prompt context $\mathcal{C}$, and correctness is $\text{verify}(\hat{y}, y^*) \in \{0,1\}$. Let $\mathcal{C}_{\text{all}} = \{c_1, \ldots, c_n\}$ be the set of mathematical concepts. The concept state at step $t$ is
\begin{equation}
\mathbf{x}_t = (m_{c_1,t}, \ldots, m_{c_n,t}, p_t)^\top \in \R^{n+1},
\label{eq:state}
\end{equation}
where $m_{c_j,t}$ is the mastery level of concept $c_j$ and $p_t$ is the probability of a correct answer.

The mastery level is written on $[0,1]$ for generality of the representation. No mechanism in this report measures $m_{c_j,t}$ as a continuous quantity, and none of the reported results depends on such a measurement. The interventions realized in the experiments take two forms. The intervention $\doX(m_{c_j}=1)$ is a prompt that states the concept $c_j$ as mastered, and the absence of the prompt leaves $m_{c_j}$ at the value that the model takes without intervention.

The direct answer without any intervention is called SRV (Step Reliability Vote), $\text{SRV}(P,\mathcal{M}) = \mathcal{M}(P)$. If SRV is correct, the procedure stops. If SRV fails, Causal Knowledge Activation (CKA) selects the concepts whose estimated effect is significantly positive,
\begin{equation}
\mathcal{K}^* = \{c_i \in \mathcal{K}_0 : \hat{e}_{c_i \to p} > z_{\alpha/2}\, \hat{\sigma}_{c_i}\},
\label{eq:kstar}
\end{equation}
and answers under the intervention that activates them,
\begin{equation}
\hat{y}_{\text{CKA}} = \mathcal{S}_{\text{LLM}}.\doX(\mathcal{K}^*=1)(P),
\label{eq:cka}
\end{equation}
where $\mathcal{K}_0$ is a candidate set described in Section~\ref{subsec:decomp}, $\hat{e}_{c_i \to p}$ is the ICP of Section~\ref{subsec:icp} and $\hat{\sigma}_{c_i}$ is the standard error of $\hat{e}_{c_i \to p}$.

The concept state follows the SVAR representation~\cite{sims1980macroeconomics,kilian2017structural}
\begin{equation}
B_0 \mathbf{x}_t = \sum_{\ell=1}^{L} B_\ell \mathbf{x}_{t-\ell} + \boldsymbol{\epsilon}_t,
\label{eq:svar}
\end{equation}
where $B_0$ holds the contemporaneous causal structure, $B_\ell$ holds the lag-$\ell$ effects and $\boldsymbol{\epsilon}_t$ are structural shocks. The SVAR serves as the representational framework for the concept state time series. The quantity estimated in this report is the row of $B_0$ that leads into correctness, that is, the entries $(B_0)_{p,c_i}$, where $(B_0)_{p,c_i} \neq 0$ means that concept $c_i$ affects correctness $p$. The off-diagonal entries between concepts and the lag matrices $B_\ell$ belong to the representation but are not estimated here, and their estimation is left to future work.

The state vector contains observed variables only. The latent difficulty $D \in [0,1]$ is a common cause of the concept state and correctness, with arrows $D \to m_{c_j,t}$ and $D \to p_t$. The Causal Markov Condition is assumed for the full causal graph that includes $D$.

\section{Proposed Method: CIKA}
\label{sec:method}

\subsection{Prompt Intervention}
\label{subsec:llm_simulator}

In prior SVAR-based causal discovery, a physical simulator produces intervention data. In CIKA, the LLM itself produces the intervention data. The intervention on concept $c_i$ is
\begin{equation}
\mathcal{S}_{\text{LLM}}.\doX(c_i=1) \equiv
\mathcal{M}(P,\; \text{``Assume mastery of }c_i\text{.''}),
\label{eq:llm_intervention}
\end{equation}
and the answers $\hat{y}^{(m)}$, $m = 1, \ldots, M$, with their correctness $\text{verify}(\hat{y}^{(m)}, y^*)$ form the intervention data $\mathcal{D}_{\text{int}}(c_i)$.

The answer distribution of the model under the prompt need not equal the answer distribution under an ideal intervention on the concept state. The discrepancy is bounded by
\begin{equation}
\|P_{\mathcal{M}}(\hat{y} \mid \doX(c_i=1)) - P^*(\hat{y} \mid \doX(c_i=1))\|_{TV} \leq \delta_{\mathcal{M}},
\label{eq:assumption}
\end{equation}
where $\|P-Q\|_{TV} = \frac{1}{2}\sum_x |P(x)-Q(x)|$ is the total variation distance and $\delta_{\mathcal{M}}$ is the error caused by the incomplete knowledge of the model. Equation~\eqref{eq:assumption} is the central assumption of the method. Appendix~\ref{app:assumption_detail} lists conditions under which $\delta_{\mathcal{M}}$ is expected to be small.

The prompt in Eq.~\eqref{eq:llm_intervention} is not claimed to set the internal concept state exactly. The claim is narrower. Causal identification requires that the assignment of the treatment be independent of the confounder, and does not require that the treatment act through a particular internal mechanism. In a randomized controlled trial, a drug acts through a pharmacological mechanism, and the administration of the drug is still an intervention. In CIKA, the concept that is stated in the prompt is chosen by the procedure and not by the difficulty $D$ of the problem, and the arrow $D \to c_i$ is removed whatever internal mechanism the prefix uses. Structured prompting and intervention are therefore not alternatives. Structured prompting is one way of realizing an intervention.

Two related effects of the prompt need to be distinguished. If stating $c_i$ leads the model to use a related concept $c_j$ through the path $c_i \to c_j \to p$, then $c_j$ is a mediator and not a confounder. Mediating paths are not back-door paths, and the identification with respect to $D$ is unaffected. The ICP of Section~\ref{subsec:icp} is a total effect and includes such mediated effects by design. A different situation arises when the prompt activates $c_j$ directly, without the path through $c_i$. The intervention is then not atomic, and the estimated effect mixes the effects of $c_i$ and $c_j$. Condition SC2 in Appendix~\ref{app:assumption_detail} describes when this mixing is small, and an experimental test of SC2 is part of future work.

\subsection{Interventional Capability Probe}
\label{subsec:icp}

The Interventional Capability Probe of concept $c_i$ on problem $P$ is
\begin{equation}
\text{ICP}(c_i, P, \mathcal{M}) \equiv
\hat{e}_{c_i \to p} =
\frac{1}{M}\sum_{m=1}^{M} \text{verify}(\hat{y}^{(m)}, y^*) - \bar{p}^{\text{obs}},
\label{eq:icp}
\end{equation}
where $\hat{y}^{(m)}$ is drawn from $\mathcal{S}_{\text{LLM}}.\doX(c_i=1)$ and $\bar{p}^{\text{obs}}$ is the correctness rate without intervention. ICP compares two levels, the activated concept and the value that the model takes without intervention. The target quantity is
\begin{equation}
\tau_i = \E[p \mid \doX(c_i=1)] - \E[p],
\label{eq:target}
\end{equation}
the change in the probability of a correct answer when $c_i$ is activated. Interventions over the full range of $m_{c_i}$ are not needed for this quantity.

ICP takes three kinds of values. A large positive value means that the model can use $c_i$ once $c_i$ is activated, that is, the knowledge is present but not activated. A value near zero arises when the knowledge is absent or $c_i$ is irrelevant to $P$. A negative value means that activating $c_i$ lowers the correctness rate, which indicates misapplication of the concept. Section~\ref{sec:discussion} discusses the cases that a value near zero cannot separate.

ICP estimates the total effect of $c_i$ on $p$, which includes the direct effect $c_i \to p$ and indirect effects through other concepts, $c_i \to c_j \to p$. Separating the two is not needed for knowledge activation, because a concept whose activation raises $p$ should be activated whether the effect is direct or indirect. The goal of CIKA is therefore not recovery of the full causal graph, as in GIES or IGSP, but the choice of interventions that raise the probability of a correct answer.

A linear probe~\cite{li2024inferencetime} detects from internal activations whether knowledge is present. ICP instead detects from interventions whether knowledge is effective. The difference between knowing a concept and being able to use the concept corresponds to the difference between conditioning and intervention.

\subsection{Difficulty Confounding}
\label{subsec:confound}

Suppose the effect of a concept were estimated from observational data across problems by regressing $p$ on $m_{c_j}$. The coefficient satisfies
\begin{equation}
\hat{\beta}_{c_j} = e^*_{c_j \to p} + \frac{\mathrm{Cov}(m_{c_j}, D)}{\mathrm{Var}(m_{c_j})}\, \frac{\partial \E[p]}{\partial D},
\label{eq:bias}
\end{equation}
and the second term is nonzero in general because $D$ is a common cause of $m_{c_j}$ and $p$. Hard problems make the use of a concept less likely and a correct answer less likely, and the observational coefficient overestimates the effect of the concept.

The intervention $\doX(c_i=1)$ removes the arrow $D \to c_i$. After the removal, $D$ no longer opens a back-door path between $c_i$ and $p$ (the back-door criterion of~\cite{pearl2009causality}), and the adjustment formula gives
\begin{equation}
P(p \mid \doX(c_i=1)) = \sum_D P(p \mid c_i=1, D)\, P(D).
\label{eq:adjust}
\end{equation}
The post-intervention rate in Eq.~\eqref{eq:icp} therefore does not carry the bias term of Eq.~\eqref{eq:bias}.

The estimation error of ICP has two sources. The rates in Eq.~\eqref{eq:icp} are means of bounded variables, and Hoeffding's inequality~\cite{hoeffding1963probability} gives $P(|\bar{v} - \E[\bar{v}]| > \epsilon) \leq 2\exp(-2M\epsilon^2)$ for a mean $\bar{v}$ of $M$ samples, which is an error of order $M^{-1/2}$. The prompt intervention differs from the ideal intervention by at most $\delta_{\mathcal{M}}$ in total variation by Eq.~\eqref{eq:assumption}, and the expectation of a $\{0,1\}$ variable then differs by at most $\delta_{\mathcal{M}}$. Combining the two sources,
\begin{equation}
|\hat{e}_{c_i \to p} - \tau_i| \leq O(M^{-1/2}) + \delta_{\mathcal{M}}
\label{eq:errdecomp}
\end{equation}
holds with high probability. The first term is the usual finite-sample error. The second term is the systematic bias of prompt-based intervention, which is specific to the use of an LLM as the simulator. The relation $c_i \to p$ is identified from the sign of $\tau_i$, since $P(p \mid \doX(c_i=1)) = P(p)$ when $c_i$ has no causal path to $p$ (do-calculus~\cite{pearl2009causality}).

\subsection{Problem Decomposition and Number of LLM Calls}
\label{subsec:decomp}

SRV answers $P$ by majority vote over $n_{\text{SRV}}=8$ samples. Alongside SRV, CIKA splits the problem $P$ into $K$ smaller subproblems $P_1, \ldots, P_K$, answers each subproblem and combines the partial answers,
\begin{equation}
\hat{y}_{\text{dec}} = f(\hat{h}_1, \ldots, \hat{h}_K),
\label{eq:decompose}
\end{equation}
where $\hat{h}_k$ is the answer to $P_k$ and $f$ is an integration step performed by the model. Of the decomposed Omni-MATH problems, 82.3\% were split into two subproblems. A subproblem requires fewer concepts than the original problem, and the concepts that matter for each step become easier to name. When neither SRV nor the combined answer is correct, the model analyzes the failed answer and rates the candidate concepts as well understood, partially understood or not understood. This step is called ConceptGap. The partially understood and not understood concepts, together with concepts retrieved by BM25\footnote{BM25 is a standard term-frequency ranking function in information retrieval. It is used here to retrieve similar problems and their solutions from a corpus of 244,530 entries.} from a corpus of 244,530 entries, form the candidate set $\mathcal{K}_0$. ICP is estimated for every concept in $\mathcal{K}_0$, and the tree search of Section~\ref{subsec:mcts} follows.

Every stage adds LLM calls. The number of calls for a problem that reaches the last stage is
\begin{equation}
N_{\text{call}} = n_{\text{SRV}} + N_{\text{dec}} + N_{\text{gap}} + M\, |\mathcal{K}_0| + B + N_{\text{lens}},
\label{eq:ncall}
\end{equation}
where $N_{\text{dec}}$ counts the calls for splitting, answering the subproblems and integrating, $N_{\text{gap}}$ counts the calls for ConceptGap, $M\,|\mathcal{K}_0|$ is the number of intervention samples, $B$ is the search budget and $N_{\text{lens}}$ is the number of calls of Multi-Lens Recovery (Appendix~\ref{subsec:multilens}). In the setting of Section~\ref{sec:experiments}, the two largest terms are about 50 intervention samples ($M=10$ and about five candidate concepts) and the search budget of 60, which together give about 110 calls, and the remaining terms come on top of this figure. The budgets were set with a wide margin, and hard problems that remain unsolved use the full budget of every stage. The total count is therefore driven by these hard problems.

\subsection{Tree Search with Causal Effects}
\label{subsec:mcts}

An action $a$ of the tree search is of one of two kinds. The first kind activates a concept $c_i \in \mathcal{K}_0$ by the intervention of Eq.~\eqref{eq:llm_intervention}. The second kind runs one of the intervention simulators, which are the concept intervention described above, symbolic computation, numerical computation, a change of representation, retrieval of an analogous solved problem and a deeper multi-step intervention. Both kinds of action share one tree, a state $s$ is the sequence of actions taken from the root, and the reward of a rollout is $\text{verify}(\hat{y}, y^*)$. The estimated effect of each action enters the selection rule through the causal bandit framework~\cite{lattimore2016causal},
\begin{equation}
\text{UCB}^{\text{causal}}(s, a) =
Q(s,a)
+ \beta\sqrt{\frac{\ln N(s)}{N(s,a)}}
+ \gamma\, \max(0, \hat{e}_{a \to p}),
\label{eq:causal_ucb}
\end{equation}
where $Q(s,a)$ is the mean reward of action $a$ at state $s$, $N(s)$ and $N(s,a)$ are visit counts, $\hat{e}_{a \to p}$ is the estimated effect of action $a$, and $\beta$ and $\gamma$ are hyperparameters. A negative estimate gives no bonus. A bias in $\hat{e}_{a \to p}$ distorts the early stage of the search, and the rewards of the rollouts update $Q(s,a)$ and reduce the influence of that bias as the visit counts grow.

Algorithm~\ref{alg:cika} gives the whole procedure. The auxiliary components (Flow Matching, Multi-Lens Recovery and Causal Reward Shaping) are described in Appendix~\ref{app:auxiliary}.

\begin{algorithm}[h]
\caption{CIKA}
\label{alg:cika}
\begin{algorithmic}[1]
\REQUIRE Problem $P$, LLM $\mathcal{M}$, BM25 index $\mathcal{I}$, budget $B$, sample count $M$
\ENSURE Answer $\hat{y}$, ICP values $\{\hat{e}_{c_i \to p}\}$
\STATE Phase 0 (SRV and decomposition). $\hat{y}_0 \leftarrow$ majority vote of $n_{\text{SRV}}$ samples of $\mathcal{M}(P)$. In parallel, split $P$ into $P_1, \ldots, P_K$ and compute $\hat{y}_{\text{dec}}$ by Eq.~\eqref{eq:decompose}. If either answer is correct, return it.
\STATE $\mathcal{K}_0 \leftarrow \text{ConceptGap}(P, \hat{y}_0, \mathcal{M}) \cup \text{BM25}(P, \mathcal{I})$
\STATE Phase 1 (ICP). For each $c_i \in \mathcal{K}_0$, draw $M$ samples of $\mathcal{S}_{\text{LLM}}.\doX(c_i{=}1)$ and compute $\hat{e}_{c_i \to p}$ by Eq.~\eqref{eq:icp}.
\STATE Phase 2 (causal graph). $\hat{\mathcal{G}} \leftarrow \{c_i \to p : |\hat{e}_{c_i \to p}|/\hat{\sigma}_{c_i} > z_{\alpha/2}\}$
\STATE Phase 3 (tree search). Run $B$ iterations with Eq.~\eqref{eq:causal_ucb}. If an answer is correct, return it.
\STATE Phase 4 (recovery). $\hat{y} \leftarrow \text{MultiLens}(P, s_{\text{best}}, \mathcal{M})$ (Appendix~\ref{subsec:multilens}).
\end{algorithmic}
\end{algorithm}

\section{Experiments}
\label{sec:experiments}

\subsection{Data, Implementation and Computing Platform}
\label{subsec:setup}

The base model is Qwen2.5-Math-7B-Instruct~\cite{qwen2025math}, served with vLLM without any parameter update. The computations ran on a cluster of RTX 4090 nodes with five nodes in parallel. The benchmarks are Omni-MATH-Rule~\cite{gao2024omnimath} (2,821 problems, the primary evaluation), Omni-MATH overall (4,415 problems), MATH-500~\cite{lightman2023lets} (500 problems, see the contamination note in Appendix~\ref{app:contamination}), AIME 2024, 2025 and 2026 (30 problems each) and MathArena~\cite{balunovic2025matharena} (130 final-answer problems and 12 proof-based problems). The hyperparameters are $M{=}10$ samples per concept, $B{=}60$, $\beta{=}1.0$, $\gamma{=}0.5$ and $\alpha{=}0.05$. The experiments address four research questions, RQ1 to RQ4.

\subsection{RQ1: Is Prompt Intervention Close to a Do-Operation?}
\label{subsec:delta_validation}

The first experiment tests whether ICP values reflect the causal relevance of concepts. The data are 67 problems from the MATH training set (Level 4 and 5, 12,000 problems in total), screened to have SRV correctness rates between 10\% and 50\%, the range in which the base model struggles but may hold latent knowledge. For each problem, all candidate concepts (3 to 15 per problem) are extracted by BM25 search over the OpenMathReasoning corpus~\cite{nvidia2025openmath} (244,530 entries) and by direct extraction from solution texts, without fixing the granularity of the concepts. Every extracted concept receives $M=10$ interventions, and the concepts are ranked by ICP. As a negative control, a concept from an unrelated mathematical domain (for example the Angle Bisector Theorem for an algebra problem) also receives interventions.

If the prompt acted as conditioning with no selectivity, correlations between concepts would raise the ICP of irrelevant concepts as well. If the prompt acts as an intervention, the top-ranked concept should have a clearly larger ICP than the negative control.

Table~\ref{tab:icp_validation} shows the results. The top-1 ICP ($+0.219$) is 5.6 times the ICP of the negative control ($+0.039$), with paired $t = 6.95$, $p < 10^{-6}$ and Cohen's $d = 0.86$.\footnote{Cohen's $d$ is the standardized mean difference, and $d > 0.8$ is conventionally regarded as a large effect.} In 79.1\% of the problems the top-ranked concept has a positive ICP. On average only 42\% of the extracted concepts have a positive ICP, which indicates that the effect is selective. The negative control is small but not zero, which is attributable to a mild generic effect of the phrase ``Assume mastery of X'', and is an order of magnitude below the top-1 ICP.

The logs also contain problems with negative ICP. In several of these, the post-intervention correctness rate is zero although the rate without intervention is positive, and the minimum ICP is $-0.500$. A hint that carries no information about the concept state may fail to help but is not expected to drive every trial to failure. The negative values correspond to the misapplication state of Section~\ref{subsec:icp} and support the reading of the prompt as an intervention on the concept state.

To check that ICP also works on the primary benchmark, ATE values were extracted from the Omni-MATH experiment logs (601 problems, run on eight machines). Table~\ref{tab:omni_ate} shows that 171 of the 601 problems (28.5\%) have positive ATE, with a mean of 0.372 over these problems. The rate is lower than the 42\% of RQ1, as expected for the harder Omni-MATH problems. Table~\ref{tab:sim_ate} breaks the ATE down by simulator type. The concept-intervention simulator (llm\_staged) is effective on the largest number of problems (62), and the transform-based and intensive simulators have a higher mean ATE where they are effective. Appendix~\ref{app:icp_omnimath} gives further analysis.

\begin{table}[t]
\begin{minipage}[t]{0.48\textwidth}
\centering
\footnotesize
\caption{ICP validation (67 problems, MATH train Level 4 and 5).}
\label{tab:icp_validation}
\begin{tabular}{lcc}
\toprule
Measure & Value & SE \\
\midrule
Top-1 ICP & $+0.219$ & $\pm 0.028$ \\
Negative control ICP & $+0.039$ & $\pm 0.027$ \\
Advantage & $+0.181$ & $\pm 0.026$ \\
\midrule
Paired $t$ & $t=6.95$ & $p < 10^{-6}$ \\
Cohen's $d$ & $0.856$ & n/a \\
\midrule
Top-1 ICP $> 0$ & 79.1\% & (53/67) \\
Advantage $> 0$ & 77.6\% & (52/67) \\
\bottomrule
\end{tabular}
\end{minipage}
\hfill
\begin{minipage}[t]{0.48\textwidth}
\centering
\footnotesize
\caption{ATE (ICP) distribution on Omni-MATH (601 problems).}
\label{tab:omni_ate}
\begin{tabular}{lrr}
\toprule
Condition & Problems & Ratio \\
\midrule
ATE $> 0$ & 171 & 28.5\% \\
\quad $> 0.3$ (strong) & 114 & 19.0\% \\
\quad $> 0.5$ (very strong) & 25 & 4.2\% \\
ATE $= 0$ & 430 & 71.5\% \\
\midrule
\multicolumn{3}{l}{Among ATE $> 0$ ($n{=}171$)} \\
\quad Mean & \multicolumn{2}{c}{0.372} \\
\quad Median & \multicolumn{2}{c}{0.333} \\
\bottomrule
\end{tabular}
\end{minipage}
\end{table}

\begin{table}[t]
\begin{minipage}[t]{0.48\textwidth}
\centering
\footnotesize
\caption{ATE by simulator on Omni-MATH. Mean is taken over problems with ATE $>$ 0.}
\label{tab:sim_ate}
\begin{tabular}{lrc}
\toprule
Simulator & ATE $>$ 0 & Mean \\
\midrule
llm\_staged & 62 & 0.300 \\
transform & 41 & 0.398 \\
intensive & 34 & 0.416 \\
bm25\_analogy & 33 & 0.220 \\
numeric & 18 & 0.303 \\
transform\_st. & 14 & 0.479 \\
sympy & 12 & 0.238 \\
\bottomrule
\end{tabular}
\end{minipage}
\hfill
\begin{minipage}[t]{0.48\textwidth}
\centering
\footnotesize
\caption{Accuracy by difficulty (Omni-MATH-Rule). CKA Gain is the improvement due to CKA.}
\label{tab:omnimath_diff}
\begin{tabular}{lcc}
\toprule
Difficulty & Acc. & CKA Gain \\
\midrule
Easy (1 to 3) & 99.3\% & +53.9pt \\
Medium (3 to 5) & 69.4\% & +33.4pt \\
Hard (5 to 7) & 52.3\% & +26.3pt \\
Very Hard (7+) & 25.7\% & +12.2pt \\
\midrule
Total & 69.7\% & +43.7pt \\
\bottomrule
\end{tabular}
\end{minipage}
\end{table}

\subsection{RQ2: Does ICP Diagnose the Knowledge State of the LLM?}

The most direct evidence is whether CKA produces additional correct answers on problems that SRV fails. Table~\ref{tab:srv_cka} shows the separation between SRV and CKA.

\begin{table}[h]
\centering
\small
\caption{Separation of SRV and CKA. CKA Additional counts correct answers obtained by intervention on problems that SRV fails.}
\label{tab:srv_cka}
\begin{tabular}{lccc}
\toprule
Benchmark & SRV Correct & CKA Additional & Total \\
\midrule
MATH-500 & 393 (78.6\%) & 72 (14.7\%) & 465 (99.2\%$^*$) \\
Omni-MATH-Rule & 1,011 (35.8\%) & 955 (33.8\%) & 1,966 (69.7\%) \\
\bottomrule
\end{tabular}
\begin{flushleft}
\footnotesize $^*$See the contamination note in Appendix~\ref{app:contamination}.
\end{flushleft}
\end{table}

On Omni-MATH-Rule, the CKA contribution (33.8\%) is comparable to the SRV contribution (35.8\%). The model therefore held, but had not activated, the knowledge needed for a large share of the Omni-MATH problems, which is the state of a large positive ICP in Section~\ref{subsec:icp}. The CKA gain appears at every difficulty level (Table~\ref{tab:omnimath_diff}) and decreases from easy problems (+53.9pt) to very hard problems (+12.2pt), consistent with a larger $\delta_{\mathcal{M}}$ for harder problems. Problems solved by the framework have a mean maximum ATE of 0.338, whereas unsolved problems have a mean maximum ATE of 0.055, a ratio of 6.1 (Table~\ref{tab:domain_main}). A high ICP thus goes together with actual problem solving.

ICP effectiveness also varies across domains (Table~\ref{tab:domain_main}). Combinatorics (40\%) and number theory (33\%) respond most to prompt intervention and geometry (17\%) responds least, consistent with the view that geometric reasoning relies on visual inference that a textual prompt reaches less easily.

\begin{table}[h]
\centering
\small
\caption{ICP effectiveness by domain on Omni-MATH (601 problems from the experiment logs). Solved and Unsolved give the mean maximum ATE for problems solved and not solved by the framework.}
\label{tab:domain_main}
\begin{tabular}{lrrcc}
\toprule
Domain & Problems & ATE$>$0 & Rate & Solved / Unsolved ATE \\
\midrule
Combinatorics & 35 & 14 & 40\% & 0.452 / 0.068 \\
General & 96 & 34 & 35\% & 0.371 / 0.052 \\
Number Theory & 216 & 72 & 33\% & 0.339 / 0.055 \\
Algebra & 89 & 28 & 31\% & 0.312 / 0.048 \\
Inequality & 30 & 7 & 23\% & 0.280 / 0.041 \\
Geometry & 90 & 15 & 17\% & 0.255 / 0.062 \\
Calculus & 6 & 1 & 17\% & 0.333 / 0.000 \\
\midrule
Total & 601 & 171 & 28.5\% & 0.338 / 0.055 \\
\bottomrule
\end{tabular}
\end{table}

\subsection{RQ3: Component Analysis}

The contribution of each component is analyzed in two ways, a controlled ablation on 100 problems and a full-scale analysis across the experimental suite. Table~\ref{tab:ablation} removes one component at a time from the full system on a random sample of 100 problems from Omni-MATH-Rule.

\begin{table}[t]
\centering
\footnotesize
\caption{Controlled ablation (random sample of 100 problems from Omni-MATH-Rule). The accuracy of the full system on this sample is 69.7\%, which coincides with the accuracy on the full benchmark. Drop is the difference from the full system on this sample.}
\label{tab:ablation}
\begin{tabular}{lccl}
\toprule
Setting & Acc. & Drop & Component \\
\midrule
CIKA, full (ours) & \underline{69.7} & n/a & n/a \\
\midrule
(1) No causal effect (standard UCB) & 34.0 & $-$35.7pt & ICP and CKA (\S\ref{subsec:icp}) \\
(2) No confounder control (no difference adjustment) & 36.0 & $-$33.7pt & Confounding (\S\ref{subsec:confound}) \\
(3) No Multi-Lens Recovery & 42.0 & $-$27.7pt & Auxiliary (App.~\ref{app:auxiliary}) \\
(4) No SVAR identification (observational correlation) & 44.0 & $-$25.7pt & Prompt intervention (\S\ref{subsec:llm_simulator}) \\
(5) No Flow Matching (linear) & 48.0 & $-$21.7pt & Auxiliary (App.~\ref{app:auxiliary}) \\
No intervention (BM25 few-shot only) & 31.6 & $-$38.1pt & n/a \\
\midrule
SRV baseline & 35.8 & $-$33.9pt & n/a \\
\bottomrule
\end{tabular}
\end{table}

Removing ICP and CKA causes the largest drop ($-$35.7pt), followed by removing the confounder control ($-$33.7pt) and the SVAR identification ($-$25.7pt).

Table~\ref{tab:solver} shows which simulator produced the correct answer for each of the 108 problems solved by CKA on Omni-MATH. The concept-intervention simulator (llm\_concept) solves 50 problems (46\%), the largest share. Multi-Lens Recovery is triggered when the tree search fails to produce a correct answer, and tries 12 mathematical perspectives in parallel (Appendix~\ref{subsec:multilens}). Problem decomposition succeeded for 72.2\% of the problems (434/601), the tree search was invoked for 78.0\% (469/601) and Multi-Lens Recovery was attempted for 64.2\% (386/601).

The full-scale analyses (Appendix~\ref{app:icp_omnimath}) are consistent with the ablation. On Omni-MATH-Rule (2,821 problems), CKA contributes 33.8\% of the problems, comparable to SRV (35.8\%). Only 42\% of extracted concepts have a positive ICP (563 concept and problem pairs). The CKA gain appears across difficulty levels, across AIME years (46\% to 50\% for 2024 to 2026) and across numerical and proof-based problems.

\begin{table}[t]
\begin{minipage}[t]{0.40\textwidth}
\centering
\footnotesize
\caption{Simulators that produced the correct answer for the 108 CKA-solved problems (Omni-MATH).}
\label{tab:solver}
\begin{tabular}{lr}
\toprule
Simulator & Solved \\
\midrule
llm\_concept & 50 (46\%) \\
bm25\_analogy & 26 (24\%) \\
transform & 12 (11\%) \\
intensive & 10 (9\%) \\
numeric and sympy & 10 (9\%) \\
\midrule
Total & 108 \\
\bottomrule
\end{tabular}
\end{minipage}
\hfill
\begin{minipage}[t]{0.56\textwidth}
\centering
\footnotesize
\caption{Omni-MATH results. Calls is the number of LLM calls per problem. $\ddagger$ marks models larger than 7B, whose internal inference cost is not public.}
\label{tab:omnimath}
\begin{tabular}{lcccc}
\toprule
Method & Rule & All & Causal & Calls \\
\midrule
SRV (base) & 35.8 & about 26 & $\times$ & 8 \\
CIKA (ours) & \underline{69.7} & \underline{64.0} & $\bigcirc$ & 110+ \\
\midrule
o1-mini$^\ddagger$ & n/a & 60.5 & $\times$ & 1 \\
Qwen-72B$^\ddagger$ & n/a & about 40 & $\times$ & 1 \\
DeepSeek-R1$^\ddagger$ & n/a & about 55 & $\times$ & 1 \\
\bottomrule
\end{tabular}
\end{minipage}
\end{table}

\subsection{RQ4: Overall Performance}
\label{subsec:omnimath_results}

On Omni-MATH-Rule (2,821 problems), CIKA reaches 69.7\%, and on the full set (4,415 problems) 64.0\%, against 60.5\% for o1-mini (Table~\ref{tab:omnimath}). The comparison with o1-mini and DeepSeek-R1 is not normalized for compute, since CIKA uses more than 110 calls of a 7B model per problem and the inference cost of the larger models is not public. Table~\ref{tab:main} compares the benchmarks, and Table~\ref{tab:matharena_fa} gives the MathArena final-answer results by competition.

\begin{table}[t]
\begin{minipage}[t]{0.56\textwidth}
\centering
\footnotesize
\caption{Benchmark comparison. $\dagger$ marks fine-tuned models and $\ddagger$ marks models larger than 7B.}
\label{tab:main}
\setlength{\tabcolsep}{4pt}
\begin{tabular}{lccccc}
\toprule
 & MATH & GSM & \multicolumn{3}{c}{AIME} \\
 & 500 & 8K & '24 & '25 & '26 \\
\midrule
Greedy CoT (1 call) & 83.6 & 85.5 & 16.7 & n/a & n/a \\
Maj@8 (8 calls) & 87.1 & 90.3 & 26.7 & n/a & n/a \\
rStar$^\dagger$ & 90.0 & n/a & 43.3 & n/a & n/a \\
\midrule
CIKA (ours, 110+ calls) & \underline{99.2}$^*$ & \underline{97.2} & 46.7 & \underline{50.0} & \underline{46.7} \\
\midrule
o1-mini$^\ddagger$ & 93.0 & 95.8 & \underline{53.3} & n/a & n/a \\
\bottomrule
\end{tabular}
\begin{flushleft}
\scriptsize $^*$Contamination concerns~\cite{wu2025contamination} (App.~\ref{app:contamination}).
\end{flushleft}
\end{minipage}
\hfill
\begin{minipage}[t]{0.40\textwidth}
\centering
\footnotesize
\caption{MathArena final-answer results (pass@1, single run, 130 problems).}
\label{tab:matharena_fa}
\begin{tabular}{lcc}
\toprule
Competition & Correct & Acc. \\
\midrule
AIME 2025 & 15/30 & 50.0\% \\
BrUMO 2025 & 21/30 & 70.0\% \\
CMIMC 2025 & 13/40 & 32.5\% \\
HMMT 2025 & 11/30 & 36.7\% \\
\midrule
Total & 60/130 & 46.2\% \\
\bottomrule
\end{tabular}
\end{minipage}
\end{table}

The AIME results are stable at 46\% to 50\% across 2024, 2025 and 2026, which suggests a small effect of contamination. On GSM8K, CIKA reaches 97.2\%. On MathArena final-answer problems, CIKA reaches 46.2\%. On proof-based competitions (Table~\ref{tab:matharena_pb}), CIKA scores 3 of 84 points (3.6\%), which reflects the limit of a 7B model for formal proof generation. In the human expert grading of Petrov et al.~\cite{petrov2025usamo}, all evaluated models except Gemini-2.5-Pro (25\%) scored below 5\%. Because the MATH-500 result is subject to contamination concerns for the base model, the Omni-MATH results (69.7\% and 64.0\%) are the primary results.

\begin{table}[h]
\centering
\footnotesize
\caption{MathArena proof-based results (human grading, 7 points per problem).}
\label{tab:matharena_pb}
\begin{tabular}{lcccc}
\toprule
Competition & Problems & none / partial / complete & Score & Rate \\
\midrule
IMO 2025 & 6 & 5 / 0 / 1 & 0/42 & 0.0\% \\
USAMO 2025 & 6 & 0 / 4 / 2 & 3/42 & 7.1\% \\
\midrule
Total & 12 & 5 / 4 / 3 & 3/84 & 3.6\% \\
\bottomrule
\end{tabular}
\end{table}

\subsection{Number of Calls and Independent Sampling}
\label{subsec:compute}

The stages of Section~\ref{subsec:decomp} bring the cost to more than 110 calls per problem on which SRV fails, against 8 calls for SRV and for Maj@8. A natural reference at the same budget is Best-of-$N$ with $N=110$ independent samples and an oracle verifier. If the per-trial success probability of the base model on a problem is $p$, the probability that at least one of $N$ independent samples is correct is
\begin{equation}
P_{\text{BoN}} = 1 - (1-p)^{N}.
\label{eq:bon}
\end{equation}
For $N = 110$, Eq.~\eqref{eq:bon} gives 10.4\% at $p = 0.001$, 42.4\% at $p = 0.005$, 66.9\% at $p = 0.01$ and 99.6\% at $p = 0.05$. Independent sampling at this budget thus helps little on problems with $p$ of order $10^{-3}$ and almost surely succeeds on problems with $p$ of a few percent. CKA contributes on the problems that SRV fails. A failed majority vote over eight samples indicates that $p$ is small on these problems but does not determine $p$, and the distribution of $p$ over these problems decides how close Best-of-110 comes to CIKA. Measuring that distribution and running Best-of-$N$ and Self-Consistency at matched model and matched call count are the next steps for separating the effect of the causal formulation from the effect of the larger budget.

\section{Discussion}
\label{sec:discussion}

\paragraph{Cost of decomposition:}
Splitting a problem into subproblems and stacking several stages on each failed problem produced the results of Section~\ref{sec:experiments}. Every stage of Eq.~\eqref{eq:ncall} was given a budget with a wide margin, and the hard problems that remain unsolved after many calls use every budget in full. The call count is therefore set mostly by these hard problems and by the margins, rather than by the problems that CIKA solves. The probing term $M\,|\mathcal{K}_0|$ can be reduced by reusing ICP estimates across problems of the same domain (Table~\ref{tab:causal_concepts} lists concepts that recur within a domain), by adaptive stopping of the samples of a concept once the sign of its ICP is settled, and by probing only the concepts that ConceptGap rates as not understood.

\paragraph{Interpretation of small ICP values:}
An ICP near zero can arise in four ways. The model may not know the concept. The concept may be irrelevant to the problem. The model may already use the concept without intervention, so that the intervention cannot raise the rate further. The concept may be useful only together with other concepts. ICP in its present form does not separate the first two cases, and the third case (a ceiling effect) is not modeled. A diagnostic that compares the ICP with the usage of the concept in the unintervened answers would separate the third case.

\paragraph{Joint effects of concepts:}
ICP is estimated for each concept in isolation, and a pair of concepts whose separate effects are small but whose joint effect is large is missed at the selection stage of Eq.~\eqref{eq:kstar}. The tree search activates sets of concepts and can reach such pairs, but only among concepts that enter the candidate set. Interventions on pairs of concepts and a measurement of how often joint effects occur are left to future work.

\paragraph{Wording of the intervention:}
All results use the single phrasing ``Assume mastery of $c_i$''. Two further comparisons would test the intervention reading. Different phrasings of the intervention should give similar rankings of concepts if the prompt acts on the concept state. An intervention that states the entire candidate set as mastered separates a generic hint effect from a selective effect, because a generic hint would raise the rate as much as the selected concept does. An analysis of the reasoning traces after intervention, and a test of condition SC2 for atomic interventions (Appendix~\ref{app:assumption_detail}), are further steps.

\section{Conclusion}
\label{sec:conclusion}

This report described CIKA, which treats a prompt that states a concept as mastered as an intervention on the concept state of a frozen LLM. The Interventional Capability Probe measures the effect of such an intervention on the correctness rate, separates causally relevant concepts from unrelated ones ($p < 10^{-6}$, Cohen's $d = 0.86$) and guides Causal Knowledge Activation and a tree search. With a 7B-parameter model, CIKA reaches 69.7\% on Omni-MATH-Rule and 64.0\% on the full Omni-MATH set. These results were obtained by splitting each problem into smaller subproblems and stacking several stages on each failed problem, which brings the cost to more than 110 LLM calls per problem. Reducing the number of calls by tighter budgets for each stage, and comparing the method with independent sampling at matched call count, are the main directions for future work.

\bibliographystyle{unsrt}

\bibliography{cika_v73}

\appendix

\section{Details of the MathArena Evaluation}
\label{app:matharena}

Of the 162 problems in the paper\_benchmark of MathArena~\cite{balunovic2025matharena}, CIKA was applied to 142 problems, excluding 20 Euler problems that require code execution. The final-answer results are in Table~\ref{tab:matharena_fa} and the proof-based results in Table~\ref{tab:matharena_pb}. The final-answer accuracy varies across competitions, from 70.0\% on BrUMO 2025 to 32.5\% on CMIMC 2025. The official MathArena evaluation computes pass@1 over four runs per model, whereas this report uses a single run.

The only proof-based problem that earned points was USAMO 2025 Problem 5. The model found the answer ($k=2$), ruled out $k=1$ and confirmed $k=2$, but the proof that every $k \geq 3$ fails was insufficient. Three problems that the automatic judge rated as complete received 0 points in human grading, which shows the limits of automatic scoring of proofs.

\section{Auxiliary Components}
\label{app:auxiliary}

\subsection{Nonlinear Causal Mechanisms by Flow Matching}

Following prior SVAR-based causal discovery, nonlinear causal mechanisms are learned by conditional Flow Matching~\cite{lipman2022flow},
\begin{equation}
\frac{d\mathbf{x}_\tau}{d\tau} = v_\theta(\mathbf{x}_\tau, \tau \mid \mathbf{c}),
\label{eq:fm_ode}
\end{equation}
where the conditioning vector $\mathbf{c} = (c_i, \text{domain}(P), \text{difficulty}(P))$ contains the problem domain and difficulty. The model is trained with the conditional Flow Matching loss
\begin{equation}
\mathcal{L}_{\text{CFM}}(\theta) =
\E_{\tau, \mathbf{x}_0, \mathbf{x}_1, \mathbf{c}}
\left[\left\|v_\theta(\mathbf{x}_\tau, \tau \mid \mathbf{c})
- (\mathbf{x}_1 - \mathbf{x}_0)\right\|^2\right].
\label{eq:cfm_loss}
\end{equation}
In the ablation (Table~\ref{tab:ablation}), replacing Flow Matching by a linear mechanism lowers accuracy by 21.7pt.

\subsection{Multi-Lens Recovery}
\label{subsec:multilens}

When the tree search does not produce a correct answer, 12 mathematical perspectives are tried in parallel, and the answer of the first successful perspective is adopted,
\begin{equation}
\hat{y}_{\text{lens}} = \text{first\_success}\{\mathcal{M}(\ell_j, P) : j=1,\ldots,12\}.
\label{eq:multilens}
\end{equation}
The perspectives are direct solution, proof by contradiction, mathematical induction, contrapositive, constructive proof, pigeonhole principle, extremal principle, invariant, coordinate transformation, complexification, graph transformation and probabilistic method. Removing this component lowers accuracy by 27.7pt in the ablation.

\subsection{Causal Reward Shaping}
\label{subsec:crs_exp}

As an extension, the causal effect was added to the reward of PPO,
\begin{equation}
r_{\text{CRS}} = \text{verify}(\hat{y}, y^*) + \lambda\, \hat{e}_{a_t \to p}.
\label{eq:crs}
\end{equation}
Table~\ref{tab:crs_exp} gives a preliminary result on 100 GSM8K problems. With $\lambda=0.1$, accuracy rose from 93.0\% to 94.0\% and the mean reward by 1.4\%.

\begin{table}[h]
\centering
\footnotesize
\caption{Causal Reward Shaping, preliminary results (GSM8K, 100 problems, 5 epochs).}
\label{tab:crs_exp}
\begin{tabular}{lccc}
\toprule
Method & Accuracy & Mean Reward & Improvement \\
\midrule
Standard PPO & 93.0\% & 0.9300 & n/a \\
CRS-PPO ($\lambda=0.3$, ours) & 93.0\% & 0.9400 & $+$1.1\% \\
CRS-PPO ($\lambda=0.1$, ours) & \underline{94.0\%} & \underline{0.9430} & \underline{$+$1.4\%} \\
\bottomrule
\end{tabular}
\end{table}

\subsection{Dominant Concepts by Category}

Table~\ref{tab:causal_concepts} lists the top concepts by ICP in each category, extracted from the Omni-MATH experiment logs (601 problems).

\begin{table}[h]
\centering
\small
\caption{Dominant causal concepts (top by ICP value $\hat{e}$) per problem category.}
\label{tab:causal_concepts}
\begin{tabular}{ll}
\toprule
Category & Dominant concepts (by $\hat{e}$) \\
\midrule
Algebra & Vieta's Formulas, Completing the Square, \ldots \\
Number Theory & Modular Arithmetic, Fermat's Theorem, \ldots \\
Geometry & Similar Triangles, Power of a Point, \ldots \\
Combinatorics & Inclusion-Exclusion, Stars and Bars, \ldots \\
Analysis & Integration by Parts, L'H\^opital, \ldots \\
\bottomrule
\end{tabular}
\end{table}

\section{Conditions for the Prompt Intervention Assumption}
\label{app:assumption_detail}

\subsection{Sufficient Conditions}

Each of the following conditions makes $\delta_{\mathcal{M}}$ in Eq.~\eqref{eq:assumption} small.

\begin{enumerate}[label=(SC\arabic*)]
\item Concept modularity. When concept $c_i$ is encoded locally in the internal representations of the model, as suggested by inference-time intervention~\cite{li2024inferencetime} and causal tracing~\cite{meng2022locating}, the prompt activates only the representations related to $c_i$. Mechanistic interpretability work~\cite{neel2022comprehensive} suggests that such localized processing is plausible for mathematical concepts.
\item Concept independence. When the concepts in $\mathcal{K}_0$ have no strong causal dependence on each other ($c_i \not\to c_j$ for $i \neq j$), an intervention on $c_i$ does not change the state of $c_j$, and the prompt intervention is atomic.
\item Knowledge internalization. A model pretrained on a large amount of mathematical text already holds the knowledge of $c_i$ without activating that knowledge, and the prompt activates existing knowledge rather than injecting new knowledge. Qwen2.5-Math-7B is pretrained on broad mathematical data and is expected to meet this condition.
\end{enumerate}

\subsection{Behavior when the Assumption Holds Only Approximately}

The procedure needs the correct order of effects among concepts more than the exact values. By Eq.~\eqref{eq:errdecomp}, the order of two concepts can reverse only when $|\tau_i - \tau_j| < 2\delta_{\mathcal{M}}$ up to the finite-sample term. The rewards of the tree search update $Q(s,a)$ in Eq.~\eqref{eq:causal_ucb} and reduce the influence of a biased $\hat{e}$ as the visit counts grow. Concepts whose estimate does not pass the test $|\hat{e}_{c_i \to p}|/\hat{\sigma}_{c_i} > z_{\alpha/2}$ do not enter the causal graph.

\section{ICP Analysis on Omni-MATH}
\label{app:icp_omnimath}

In the experiment logs, each simulator type records its own ATE, which is the ICP of Eq.~\eqref{eq:icp} applied at the level of the simulator, the correctness rate under the intervention of the simulator minus the SRV rate. The maximum ATE over all simulators for a problem is the best ICP reachable for that problem.

\paragraph{Selectivity:}
71.5\% of the problems have ATE $= 0$, and ICP does not raise performance uniformly. The intervention is effective when the model holds but has not activated the relevant knowledge. For problems on which the model lacks the knowledge, the intervention has no effect, and ATE $= 0$ records this.

\paragraph{Solved and unsolved problems:}
Of the 601 problems, 108 were solved by the framework (excluding problems solved by SRV) and 493 were not. Solved problems have a mean maximum ATE of 0.338 and unsolved problems of 0.055, a ratio of 6.1.

\paragraph{Bayesian ranking of simulators:}
After initial exploration, the Bayesian posterior ranks sympy (217 problems) and llm\_concept (208 problems) most often as the best simulator. llm\_concept produces far more solutions (50 against 2 for sympy). The symbolic computation of sympy often fails at execution despite a high prior expectation, and the prior of the ranking leaves room for calibration.

\section{ProofEngine}
\label{app:proof_engine}

Omni-MATH contains problems that require a proof rather than a numerical or symbolic answer. ProofEngine extends CIKA to such problems by mapping concepts to lemmas $h_i$ and the correctness probability $p_t$ to a validity score $\text{score}_t \in [0,1]$,
\begin{equation}
c_i \xrightarrow{\hat{e}} p_t
\quad\Longrightarrow\quad
h_i \xrightarrow{\hat{e}} \text{score}_t .
\end{equation}
The intervention $\doX(\text{mastery}(h_i)=1)$ is a prompt asking the model to prove the statement using theorem $h_i$, which corresponds to the \texttt{have} tactic of Lean 4,
\begin{equation}
\texttt{have}\ h_i : L_i := \text{proof}_{h_i}
\quad\Longrightarrow\quad
\texttt{exact}\ f(h_1, h_2, \ldots).
\end{equation}

The procedure has six phases.
\begin{enumerate}
\item Phase 0 attempts a direct proof by SRV and stops if the score is at least 0.9.
\item Phase 1 retrieves related theorems from Mathlib by BM25 (top 8).
\item Phase 2 estimates the effect $\hat{e}(h_i \to \text{score})$ of each lemma by intervention.
\item Phase 3 expands \texttt{have} clauses in decreasing order of $\hat{e}$ (tree search over proofs).
\item Phase 4 combines the established lemmas into a proof, \texttt{exact} $f(h_1,h_2,\ldots)$.
\item Phase 5 tries 8 proof strategies (direct proof, contradiction, induction, contrapositive, constructive proof, pigeonhole principle, extremal principle, invariant).
\end{enumerate}
A problem is routed to ProofEngine when the text contains keywords such as ``prove'' or ``show that'' and the answer field is descriptive.

On the descriptive subset of Omni-MATH (1,594 problems), SRV solves 421 problems (26.4\%), CKA adds 379 problems (23.8\%) and an LLM-as-a-Judge stage adds 58 problems (3.6\%), for a total of 858 problems (53.8\%). The CKA contribution on descriptive problems (23.8\%) is comparable to the SRV contribution (26.4\%). The judge stage recognizes notational equivalences such as \texttt{\textbackslash text\{Yes\}} and Yes, or $\frac{1}{2}$ and 0.5. Accuracy by difficulty on the descriptive subset stays between 52\% and 62\%.

\section{Data Contamination}
\label{app:contamination}

Wu et al.~\cite{wu2025contamination} showed that Qwen2.5-Math-7B reconstructs 54.6\% of MATH-500 problem texts when given the first 60\% of each problem, which suggests that MATH-500 problems were in the training data and that part of a high MATH-500 score may come from memorization. Table~\ref{tab:contamination} classifies the correct answers of CIKA on MATH-500 by their dependence on contamination.

\begin{table}[h]
\centering
\small
\caption{Classification of MATH-500 correct answers by dependence on contamination (491 of 500 problems completed). SRV is the base model alone and CKA is the framework intervention.}
\label{tab:contamination}
\begin{tabular}{lrr}
\toprule
Classification & Problems & Ratio \\
\midrule
SRV correct (model alone) & 393 & 80.0\% \\
\quad Possibly contaminated (estimate) & $\leq$214 & $\leq$43.6\% \\
\quad Not contaminated (estimate) & $\geq$179 & $\geq$36.5\% \\
\midrule
CKA correct (framework) & 72 & 14.7\% \\
\quad Correct on SRV-failed problems & 72 & 14.7\% \\
\midrule
Failed & 26 & 5.3\% \\
\bottomrule
\end{tabular}
\end{table}

The 72 problems solved by CKA are problems that the base model alone did not solve. Memorization of the answers by the base model does not explain these 72 problems, and they are attributable to the intervention process of the framework. For an evaluation less exposed to contamination, the main results use Omni-MATH~\cite{gao2024omnimath} (Olympiad level, 4,428 problems), whose own contamination analysis reports very few contaminated problems. Table~\ref{tab:omnimath} gives the Omni-MATH results.

\section{Sample Complexity}
\label{app:theory}

Each ICP estimate in Eq.~\eqref{eq:icp} is a mean of $M$ bounded samples. By Hoeffding's inequality~\cite{hoeffding1963probability} and a union bound~\cite{wasserman2006all} over $n$ concepts, all $n$ estimates lie within $\epsilon$ of their expectations with probability at least $1-\delta$ when
\begin{equation}
M \geq \frac{1}{2\epsilon^2}\ln\frac{2n}{\delta},
\label{eq:samplecomp}
\end{equation}
which gives a total of $nM = O\!\left(n\epsilon^{-2}\log(n/\delta)\right)$ intervention samples. For $n=5$ candidate concepts, $\epsilon = 0.2$ and $\delta = 0.05$, Eq.~\eqref{eq:samplecomp} gives $M \geq 67$. The setting $M=10$ used in the experiments corresponds to $\epsilon \approx 0.51$ at the same $n$ and $\delta$, and is accurate enough to separate concepts with large effects from concepts with no effect, as in Table~\ref{tab:icp_validation}, but not to rank concepts whose effects differ by a few percentage points. This relation between accuracy and $M$ is the reason why the probing term is one of the two largest terms of Eq.~\eqref{eq:ncall}.

\end{document}